%% file: main.tex
\documentclass[runningheads]{llncs} 
\usepackage[T1]{fontenc}
\usepackage{amsmath}
\usepackage{amsfonts}

\usepackage[hyphens]{url}  
\usepackage{graphicx} 
\usepackage{algorithmic}
\usepackage[ruled,vlined,linesnumbered]{algorithm2e}

\title{Information Consistent Pruning: How to Efficiently Search for Sparse Networks?}

\author{Soheil Gharatappeh\inst{1} \and 
  Salimeh Yasaei Sekeh\inst{2} 
  }

\authorrunning{S. Gharatappeh et al.}

\institute{
  School of Computing and Information Science at University of Maine \and
  Department of Computer Science at San Diego State University 
    \email{soheil.gharatappeh@maine.edu},
    \email{ssekeh@sdsu.edu}
}
\usepackage{bibentry}
\usepackage[inline]{enumitem}

\newcommand{\ourmethod}{InCoP}
\newcommand{\Real}{\mathbb{R}}
\newcommand{\Ex}[2]{\mathbb{E}_{#1}\left[\, #2 \,\right]}
\newcommand{\loss}{\ell}

\begin{document}

\maketitle

\begin{abstract}
Iterative magnitude pruning methods (IMPs), proven to be successful in reducing the number of insignificant nodes in over-parameterized deep neural networks (DNNs), have been getting an enormous amount of attention with the rapid deployment of DNNs into cutting-edge technologies with computation and memory constraints.
Despite IMPs popularity in pruning networks, a fundamental limitation of existing IMP algorithms is the significant training time required for each pruning iteration.
Our paper introduces a novel \textit{stopping criterion} for IMPs that monitors information and gradient flows between networks layers and minimizes the training time.
Information Consistent Pruning (\ourmethod{}) eliminates the need to retrain the network to its original performance during intermediate steps while maintaining overall performance at the end of the pruning process.
Through our experiments, we demonstrate that our algorithm is more efficient than current IMPs across multiple dataset-DNN combinations.
We also provide theoretical insights into the core idea of our algorithm alongside mathematical explanations of flow-based IMP. Our code is available at \url{https://github.com/Sekeh-Lab/InfCoP}.

\end{abstract}

\section{Introduction}

Over-parameterized deep neural networks (DNNs) have emerged at the forefront of cutting-edge technologies, demonstrating success in areas such as computer vision\cite{krizhevsky2012imagenet}, speech recognition~\cite{hinton2012deep}, and large language models \cite{otter2020survey}. As these networks typically consist of millions of parameters, they demand significant computational power, making them unsuitable for edge devices with limited resources. Initially, researchers sought to optimize network performance by experimenting with architectures~\cite{pouyanfar2018survey} and incorporating techniques such as BatchNorm~\cite{ioffe2015batch}, dropout \cite{srivastava2014dropout}, and residual layers \cite{he2016deep}.
As networks were able to handle highly complex problems and their performance reached a plateau, the focus shifted towards optimizing network efficiency.

Pruning stands out as a prominent approach to learning subnetworks~\cite{ganesh2021mint,ganesh2024slimming}. These methods are capable of delivering comparable performance to the original dense network while significantly enhancing their training time \cite{dettmers2019sparse} and inference time \cite{luo2017thinet}. Pruning makes DNNs applicable for deployment on resource-constrained devices \cite{frankle-2018-lotter-ticket-hypot,han-2015-learn-both}. These techniques have been shown superior generalization error, attributed to their inherent regularization effect~\cite{lecun1989optimal,han-2015-learn-both,wen2016learning,you2019gate} and robustness to noises~\cite{ahmad2019can}.

Pruning methods vary from one-shot approaches~\cite{lee-2019-snip,chen2021only}, where the network is pruned in a single step, to iterative pruning techniques \cite{han-2015-learn-both,evci-2020-gradien-flow,orseau2020logarithmic} that involve a series of pruning and retraining cycles. Among these methods, the Lottery Ticket Hypothesis (LTH) \cite{frankle-2018-lotter-ticket-hypot} has emerged as a particularly successful strategy for effectively eliminating insignificant nodes while preserving the original accuracy of the network. LTH suggests that within every randomly initialized DNN, there exists a subnetwork capable of achieving performance comparable to the original network when trained in isolation \cite{frankle-2018-lotter-ticket-hypot,frankle-2019-linear-mode}.
These identified subnetworks, referred to as "winning tickets," can be systematically uncovered through an iterative process, in which a fixed $p\%$ of the network parameters are pruned, and the network is retrained to replicate the original accuracy.
While this method has shown promise, it typically demands extensive training, resulting in it being less practical for scenarios with limited resources.

Sparsity-Informed Adaptive Pruning (SAP) \cite{diao-2023-prunin-deep} tackles the sub-optimality of the pruning step in LTH by monitoring the network's sparsity level and pruning the weights accordingly. SAP introduces a novel metric called PQ-Index (PQI), a Gini-Index like sparsity metric that measures the compressibility of the network. PQI criterion determines a lower bound on the pruning rate in each iteration, therefore SAP prunes more aggressively in each iteration, achieving the desired remaining weights sooner than LTH \cite{diao-2023-prunin-deep}. However, SAP uses the same stopping criterion as the LTH and trains the network for a fixed number of epochs until it reaches the same test accuracy.

Expanding on the concept of SAP, our focus is to enhance LTH by reducing the training time via refining the stopping criterion. This involves monitoring the network's proximity to the optimal original dense network.
We evaluate this closeness in terms of \textit{gradient and information flow}. The significance of gradient flow (GF) in an effective convergence of DNNs is well-established in the context of sparse networks \cite{evci-2020-gradien-flow,wang2020picking}.
Additionally, in \cite{tessera2021keep}, the authors investigated the optimization of sparse networks from initialization, regularization, and architectural perspectives.
They also demonstrated how Effective Gradient Flow (EGF) can serve as a direct proxy for assessing network performance.
In our study, inspired by this work, we explore the use of GF for optimizing the IMP. Our findings demonstrate that by utilizing the GF as a metric of proximity to optimality, we can significantly improve the efficiency of IMPs.

Information flow (IF) has also been studied in relationship to the generalization ability of neural networks \cite{jin-2020-how-does}, and in the context of pruning deep networks in \cite{ganesh-2021-slimm-neural}.
In these studies, a form of correlation matrix between the consecutive layers of the network is utilized to determine the flow of information between them. 
We utilize the correlation between the activations of consecutive layers to measure IF through the layers in an IMP process and enhance the currently available pruning stopping criterion.

{\bf Contribution:} Our contribution in this work is twofold.
\begin{enumerate*}[label=(\roman*)]
    \item Our study challenges the conventional notion that retraining a network to its original performance is essential to the IMP process.
  We demonstrate that traditional accuracy-based stopping criteria algorithms, suitable for dense networks, do not accurately capture the layer-wise sensitivity of the network and may not accurately reflect proximity of the sparse network parameters to the optimal point.
  \item We introduce two additional metrics (IF and GF) and argue that maintaining these metrics within each iteration is adequate to sustain the performance (test accuracy) of the network.
\end{enumerate*}
\textbf{Information Consistent Pruning (\ourmethod{})} leverages GF and IF and provide a more comprehensive assessment of the resemblance between the sparse network and the fully-trained dense network throughout the layers. We demonstrate that using these metrics, \ourmethod{} efficiently prunes DNNs iteratively, using a much smaller number of epochs.


\section{Method}

\paragraph{Notation for subnetworks and training:}
LTH removes insignificant weights from the network $F_w$ through a process of sequential pruning and retraining. The network at intermediate steps is represented as $F_{w^{\tau}_t}$, with $w^{\tau}_t \in \mathbb{R}^D$ denoting the network weights at pruning iteration $t = 1, 2, \ldots, T$ and training epoch $\tau = 1, 2, \ldots, E$.
This involves sequentially training the network for $T$ iterations to achieve the target sparsity level $s$.
Let $\mathcal{A}(w^\tau_t, t')$ be the training algorithm at iteration $t$ that trains the network $F_{w^\tau_t}$ for $t'$ epochs and outputs $w^{\tau+t'}_t$, and if the network is trained to its optimal point, the training algorithm returns $w_{t}^*$.

In the first step, LTH generates a randomly initialized dense network $F_{w^0_0}$ and train this network for $E$ epochs to its optimum performance $w^*_0=\mathcal{A}(w^0_0, E)$.
Then, in each iteration, the $p\%$ least important weights from the weight vector $w_t \in \Real^d$ are selectively pruned based on their magnitude by multiplying them to a mask matrix $m_t \in \{0, 1\}^D$.
The weight matrix sparsification is carried out using $m_t \odot w_t$ operation, where $\odot$ denotes the Hadamard product. Following the pruning step, $p \cdot d_t = p \|m_t\|_0$ elements are eliminated from the network in each iteration, with $\|.\|_0$ representing the $0$-norm denoting the count of non-zero elements in a vector.

\subsection{SAP: Sparsity-Informed Adaptive Pruning}\label{sec:SAP}
SAP~\cite{diao-2023-prunin-deep} is an LTH-based algorithm that prunes and retrains the network iteratively more optimally by considering the sparsity of the network weight vectors. 
The PQI, $I_{p, q}(w)$, and $r_t$ (lower bound on the number of remaining parameters) is obtained by
\begin{equation}
  \label{eq:pqi}
  \begin{split}
  \mathrm{I}_{p, q}(w)=1-d^{\frac{1}{q}-\frac{1}{p}} \frac{\|w\|_p}{\|w\|_q}, \\ 
  r_t=d_t\left(1+\eta_r\right)^{-q /(q-p)}\left[1-\mathrm{I}\left(w_t\right)\right]^{\frac{q p}{q-p}},
  \end{split}
\end{equation}
where, $0 \le p \leq 1 \le q$,  $\|w\|_p=\left(\sum_{i=1}^d\left|w_i\right|^p\right)^{1 / p}$ is $L_p$-norm, $\eta_r$ is the smallest value such that $\sum_{i \notin M_r}\left|w_i\right|^p \leq \eta_r \sum_{i \in M_r}\left|w_i\right|^p$ and $M_r$ is the set of $r$ indices of $w$ with the largest magnitudes.
The authors of \cite{diao-2023-prunin-deep} showed that by choosing the pruning ratio adaptively and according to eq. \ref{eq:pqi} we can prune more aggressively with respect to the LTH without losing much of the performance. This is because PQI varies according to the compressibility of the network. 
SAP is very similar to LTH, except for when LTH prunes the smallest $p\%$ parameters, SAP finds $r$ and computes the number of parameters to be pruned by $c_t=\big\lfloor d_t \cdot \min \big(\gamma\big(1-\frac{r_t}{d_t}\big), \beta\big)\big\rfloor$,
where $\gamma$ and $\beta$ are two additional hyperparameters, a scaling factor and the maximum pruning ratio, respectively, to further optimize the pruning rate and avoid over-pruning. Once $c_t$ is determined, we prune the extra parameters by setting their corresponding element in masking $m_t$ to zero and setting their gradient to zero.

\subsection{Flow-based Stopping Criterion}
\ourmethod{} is based on SAP, but instead of measuring accuracy, it maintains the IF (\ourmethod{}-IF) and GF (\ourmethod{}-GF) to train more optimally. SAP and LTH randomly initialize a network and train it for $E$ epochs until they reach the maximum accuracy $Acc^*$. 
During the training step, the network undergoes $E$ training epochs in all subsequent iterations to achieve the same maximum accuracy \cite{frankle-2018-lotter-ticket-hypot}.
Our experiments show that this step is redundant and can be further optimized. By leveraging IF and GF, \ourmethod{} reaches the same final remaining weights while maintaining the performance and smaller training epochs. 
To ensure that the network is retrained to its original performance after each iteration of pruning, we harness the power of IF and GF, consistently preserving information throughout the process. 

{\bf Notations:}
Let $F_w$ be a function mapping the input space $\mathcal{X}$ to a set of classes $\mathcal{Y}$ with a set of weights $w$ , i.e. $F_w: \mathcal{X}\mapsto \mathcal{Y}$. $f^{(l)}_w$ is the output of $F_w$ for the $l$-th layer with $M_l$ as the number of filters at layer $l \in \{0, \ldots, L\}$. When feasible, we omit the subscript $w$, and such omission does not cause confusion. $f^{(l)}_i$ is the $i$-th filter in layer $l$ and $\sigma ^{(l)}$ is the activation function in layer $l$.

{\bf Information Flow: }
The IF between two consecutive layers $l$ and $l+1$ of network $F_w$ with $M_l$ and $M_{l+1}$ filters respectively is defined by $\Delta_w(f^{(l)}, f^{(l+1)}) \in \mathbb{R}^{M_l\times M_{l+1}}$. Connectivity, the $(i, j)$ element of IF matrix, is defined by

\begin{equation} \label{eq:def-connectivity}
  \Delta_w \left(f^{(l)}_i,f^{(l+1)}_j\right)= 
  \Ex{(X,Y)\sim
  \mathcal{D}}{f_{i}^{(l)}\; f_{j}^{(l+1)}} \in \mathbb{R}.
\end{equation}

{\bf Gradient Flow:}
The GF in DNN measures optimization dynamics and is typically approximated using the norm of the gradients of the network~\cite{wang2020picking,evci-2020-gradien-flow}. Studying GF is particularly relevant in the context of pruning and sparsifying networks. 
Following~\cite{tessera2021keep}, we compute GF $\mathbf{g} \in \Real^L$ by
\begin{equation}
  \label{eq:gf}
    \mathbf{g}_i =\frac{\partial \ell}{\partial w^{(i)}} \odot m^{(i)} \in \mathbb{R}^{M_l}
\end{equation}
where $\ell$ is the loss function, and $m^{(i)}$ is masking vector applied on layer $i$.


\begin{figure}
  \centering
  \includegraphics[scale=.45]{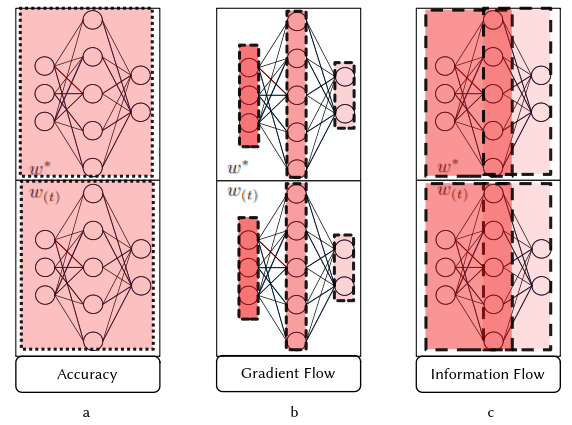}
  \caption{a. Accuracy-based stopping criterion: the overall performances of the optimal and sparse networks are compared b. gradient flow-based: the gradient flow of corresponding layers are compared (similar shades of colors are compared between the two networks) and c. information flow-based: the information flow of corresponding consecutive layers are compared (similarly, boxes with similar shades of colors are compared). }
  \label{fig:schematic}
\end{figure}

\subsection{Intuition and Method Overview}

We establish the importance of maintaining the IF and GF, $\Delta_w$ and $\mathbf{g}$, of the network throughout the process in IMP.
By comparing these metrics of the sparse network to the optimally trained dense network $\Delta_{w^*}$ and $\mathbf{g}^*$, we can effectively ensure that the network's performance is maintained during pruning.
As illustrated in Fig. \ref{fig:schematic}, the conventional accuracy-based stopping criterion algorithm relies solely on the final network layer. This narrow focus results in excessive sensitivity to the performance of this singular layer, leading to a disregard of other crucial aspects of the entire network aspects during the process.
By contrast, our proposed algorithm monitors the GF and IF across all network layers. The gradient based algorithm compares the gradient of each layer to the optimal network, and the information based algorithm looks at the interaction between layers in the current network and compare them with the optimal one.  

 Let $\Phi_w$ represent either the gradient $\mathbf{g}_w$ or the information $\Delta_w$ flow of the network $f(w)$, while $\Phi_w^*$ denotes the flow of the optimal dense network obtained via the training algorithm after training the network for $E$ epochs, $w^* = \mathcal{A}_0(w_0, E)$.
We define $S = \{w : ||\Phi_w - \Phi^*|| \leq \epsilon \}$ as the space where the IF/GF distance between the network and the optimal flow is within $\epsilon$. Additionally, let $M = \{w : || Acc_w - Acc^* || \leq \delta \}$ represent the space where the accuracy of the corresponding networks lies within a $\delta$-vicinity of the optimal point.

\begin{figure}
  \centering
  \includegraphics[width=0.85\textwidth]{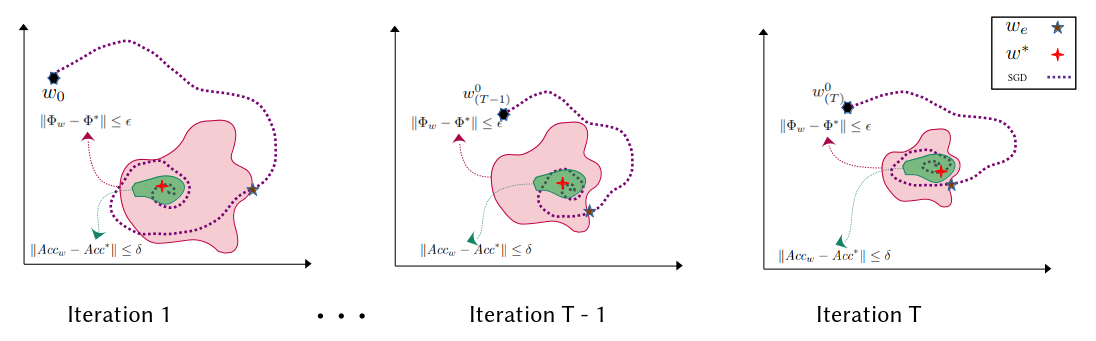}
  \caption[SGD training]{SGD in pruning with accuracy based training vs. flow based training. The red region corresponds to the set $S$, representing the information based stopping criterion, while the green region corresponds to the set $M$, symbolizing the accuracy based stopping criterion. The shift and contraction of these areas occur due to dimensionality reduction resulting from the pruning step.}
  \label{fig:sgd-pruning}
\end{figure}

An IMP algorithm starts with a fully trained network $w^*$, wherein a portion of the network is pruned in each iteration, followed by retraining the network to attain its original performance.
As illustrated in Fig. \ref{fig:sgd-pruning} the training algorithm (i.e. SGD) locates the optimal point by descending along the gradient and stops when the network's performance falls within an acceptable threshold (green region in Fig. \ref{fig:sgd-pruning})
\begin{equation}
  \label{eq:opt_M}
w_{t}^* = \mathcal{A}_M(w_{t}^0, E),
\end{equation}
where, $\mathcal{A}_M$ is the training algorithm that operates on set $M$, indicating its reliance on accuracy-based stopping criterion.
Subsequently, these optimal weights are utilized for pruning, which serves as the initial weights for the subsequent iteration.
While the LTH robustly demonstrates that iteratively following this process eventually yields a network within an acceptable range of performance that is notably smaller than the initial network, it overlooks the dynamic nature of the problem space throughout each pruning iteration.
Consequently, starting precisely from the optimal point of the preceding step (green region in Fig. \ref{fig:sgd-pruning}) becomes less critical.
Through our experiments, we highlight that by solely focusing on maintaining either the GF and IF of the network during the training process - staying within the red region as depicted in Fig. \ref{fig:sgd-pruning} and training for $e$ epochs and returning
    $w_e = \mathcal{A}_S(w_t, e),$
it is possible to ensure the preservation of the network's performance up to the last iteration of our IMP algorithm when $e$ is significantly smaller that $E$ in equation \ref{eq:opt_M}.

In the Experiments section, we empirically demonstrate that utilizing a flow-based stopping criterion does not compromise the overall network performance. Furthermore, we demonstrate that this approach significantly enhances efficiency and reduces training time, highlighting the practical benefits of our proposed methodology.
We provide a comprehensive theoretical analysis of our method in the Supplementary Materials for further insight and understanding.

{\bf Proposed Algorithm} In our proposed algorithm, we first initialize the network, either by random initialization or using the pretrained network. Then, the network is fine-tuned to an optimal point $w^*$, where $\Phi^*$ is computed and stored.
Then, we go through a sequence of pruning and retraining the network up to a certain point.
During the training phase, we have to compute $\Phi$, which can be either $\Delta$ or $\mathbf{g}$, making \ourmethod{}-IF and \ourmethod{}-GF, respectively.
We stop the training loop if the distance between $\Phi$ and $\Phi^*$ is less than $\epsilon$.
Then, the number of pruned neurons is obtained by SAP's approach, and the created mask is applied to the network's weights.

\begin{algorithm}[t!]\label{algo.1}
\caption{Information Consistent Pruning}\label{alg:fogrob}
\SetKwInOut{Input}{In}\SetKwInOut{Output}{Out}
\SetKw{continue}{continue}
\SetKw{break}{break}
\Input{
Given dataset\\
  $E$: Total \# of training epochs to obtain $Acc^*$\\
 $\epsilon\geq 0$  and $T$: total \# of iterations: Hyperparameters \\
 $\Phi$: The proximity metric, $\Delta$ for \ourmethod{}-IF, and $\mathbf{g}$ for \ourmethod{}-GF
}
  \Output{$w_T, m_T$: Weight and mask matrices at the iteration $T$ }
  \BlankLine
  Initialize $w_0$ (retrieve the pretrained model)\\
  Train the network for $E$ epochs to reach the maximum accuracy $Acc^*$\\
  Set $w^*$ the weights of the network with $Acc^*$ accuracy \\
  Compute $\Phi^*$\\
  \For{$t \gets 1$ to $T$}
  {
    \For{$e \gets 1$ to $E$}
    {
        Train the network for 1 epoch, arrive at $w_e$ \\
        Compute $\Phi_e$\\
        \If{$\| \Phi_e - \Phi^* \| \le \epsilon$,
          }
        {
          \break\\
        }
      }
      Compute sparsity level $c_t$ from Sec.\ref{sec:SAP} using PQ-Index~\cite{diao-2023-prunin-deep}\\
      Prune $c_t$ parameters based on $w_t$ and mask $m_t$, then create mask $m_{t+1}$\\ 
      Reset the remaining weights to $w_t$ by $w_{t+1} = m_{t} \odot w_t$
  }
\end{algorithm}

\section{Experimental Results}
{\bf Datasets and DNNs:} We use three primary datasets to evaluate our proposed method: Fashion-MNIST~\cite{xiao2017fashion}, MNIST \cite{deng2012mnist}, and also to showcase the effectiveness of our method on larger datasets we incorporated CIFAR-10~\cite{krizhevsky2010convolutional} into our experimental results.
We use two DNN architectures to evaluate our method in the context of IMP, VGG16 in VGG family \cite{simonyan2014very}, and ResNet18 and ResNet50 to assess our method on a wider range of network depth.
We use SGD with $0.01$ learning rate and $0.9$ weight decay as the optimizer, and the batch size is $64$.
In our implementation, we rely on Pytorch and leverage their established implementations of ResNet and VGG models. Specifically, we utilize models that have been pretrained on the ImageNet dataset. At the beginning of each experiment, we fine-tune the networks for $k=10$ epochs to optimize their performance and achieve maximum test accuracy. The resulting network state is then saved for future reference.
Throughout the pruning process, the networks in SAP, LTH, \ourmethod{} are trained for 20 training epochs, after which a pruning step is applied.

We run experiments for $T=15$ pruning iterations and compare the results of \ourmethod{} with SAP and LTH.
For \ourmethod{} and SAP we considered two sets of $(p, q)$, $(1.0, 2.0)$ and $(0.5, 1.0)$.
These two sets of pairs are suggested by \cite{diao-2023-prunin-deep} for more and less aggressive pruning regiments.
The remaining hyperparameters are selected following the guidelines provided in \cite{diao-2023-prunin-deep}.
We have $p=0.2$ for LTH throughout all experiments.
The scope of pruning we experimented with is \textit{layer-wise}, which involves pruning a specified number of neurons from each layer in the network.
The entire experiment is run for $3$ trials and the means and standard deviations of the experiments are demonstrated on the plots.



\subsection{Performance Comparison}\label{sec:performance_compare}

\begin{figure}
  \centering
    \includegraphics[scale=.20]{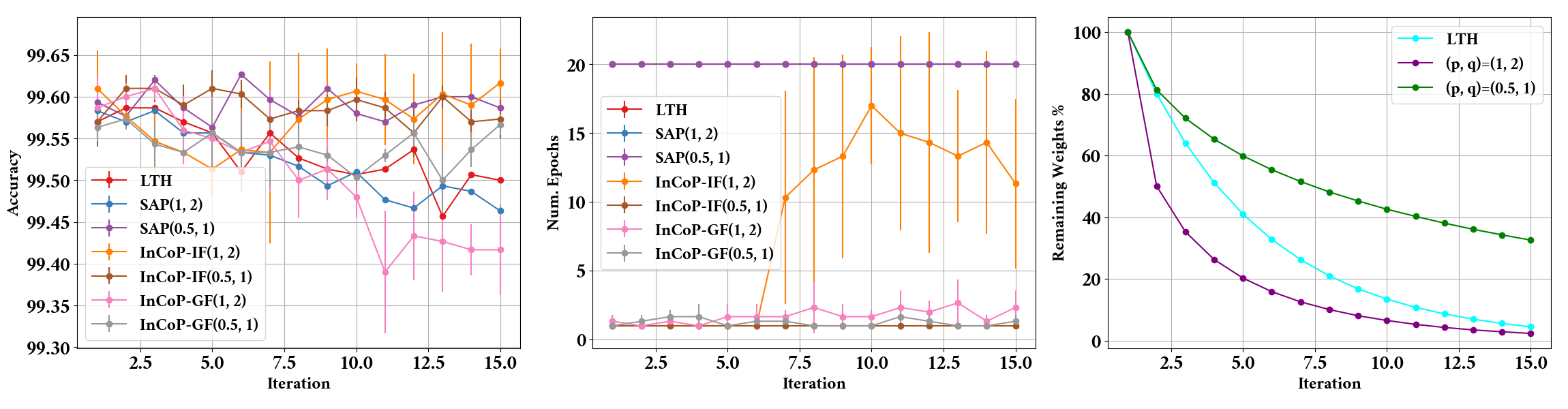}\\
  \includegraphics[scale=.20]{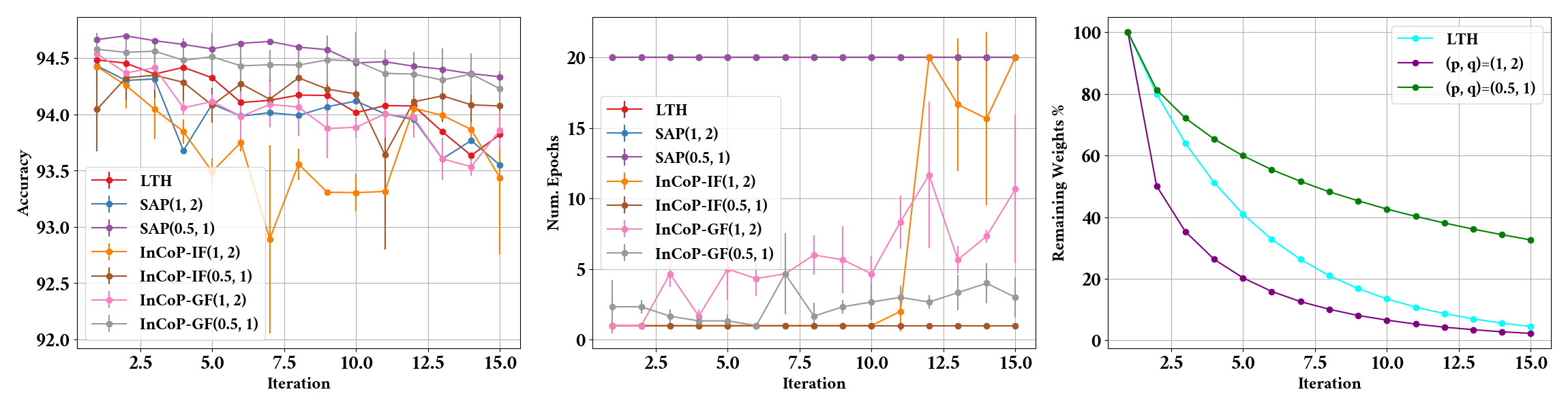}\\
  \includegraphics[scale=.20]{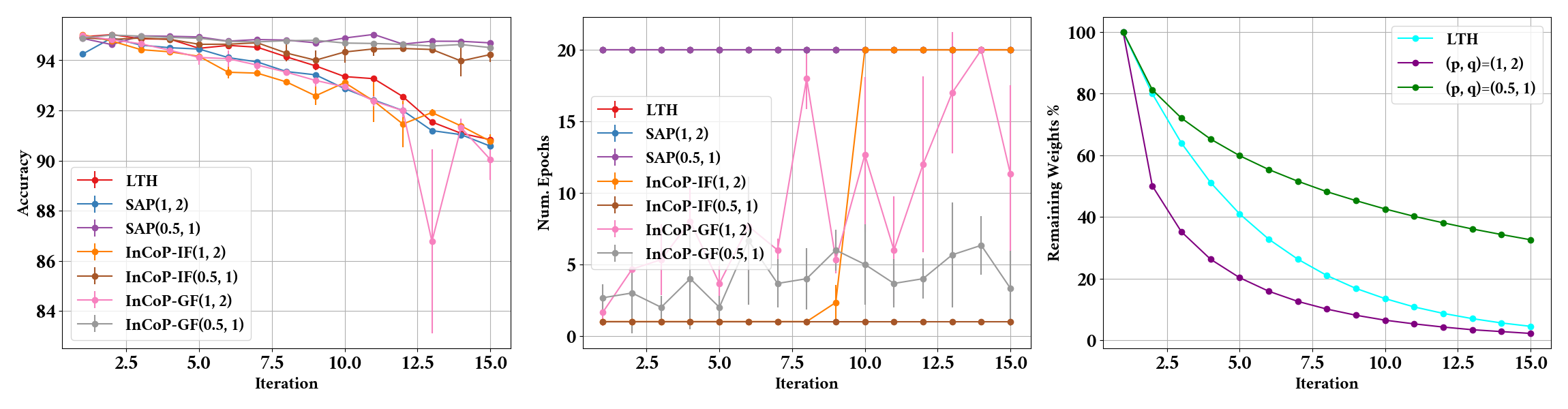}
\caption{ResNet18 - MNIST~(top), FashionMNIST~(middle), CIFAR10~(bottom). The left column shows accuracy at each iteration $t$, the number of training epochs required in $t$ is shown in middle column, and the right column shows remaining weights of $F_w^{(L)}$ in different iterations. Purple line $\rightarrow (p, q) = (1, 2)$, green $\rightarrow (p, q)=(0.5, 1)$, and cyan $\rightarrow$ LTH.
  }
\label{fig:resnet}
\end{figure}
 

\begin{figure*}[th]
  \centering
  \includegraphics[scale=.20]{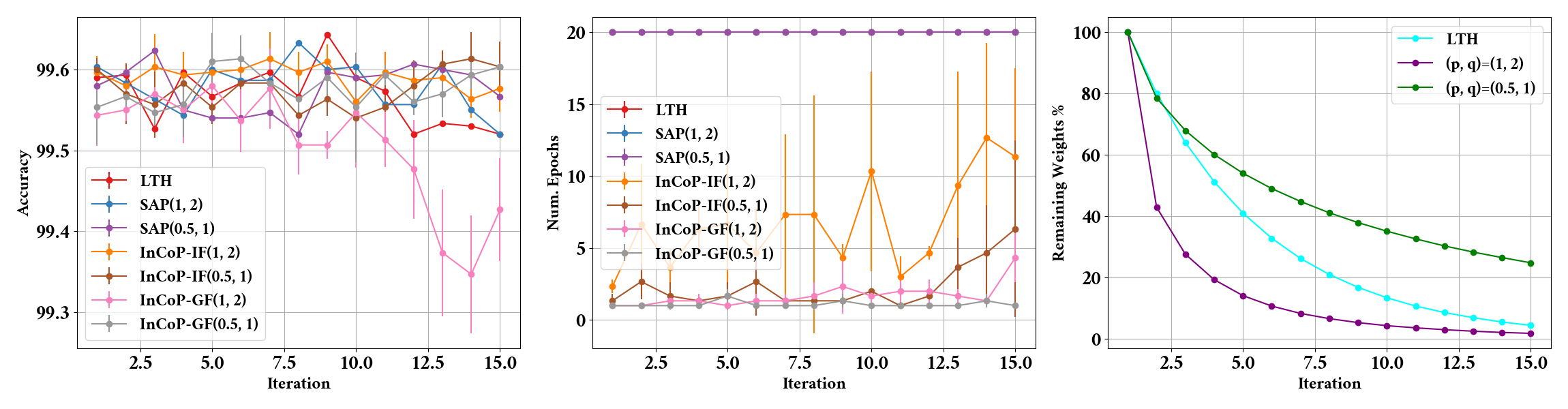}\\
  \includegraphics[scale=.20]{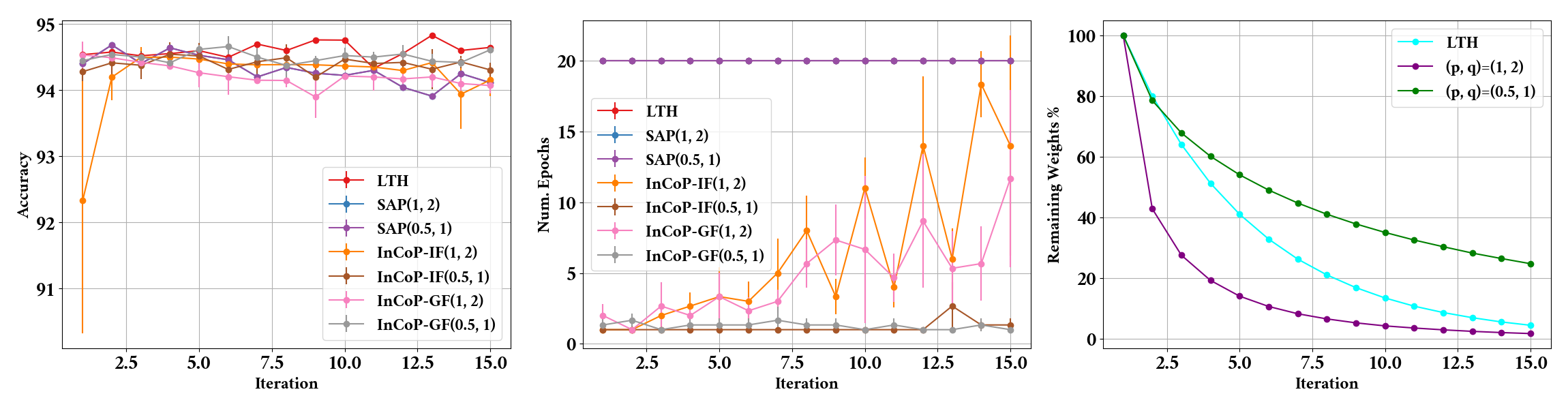}\\
  \includegraphics[scale=.20]{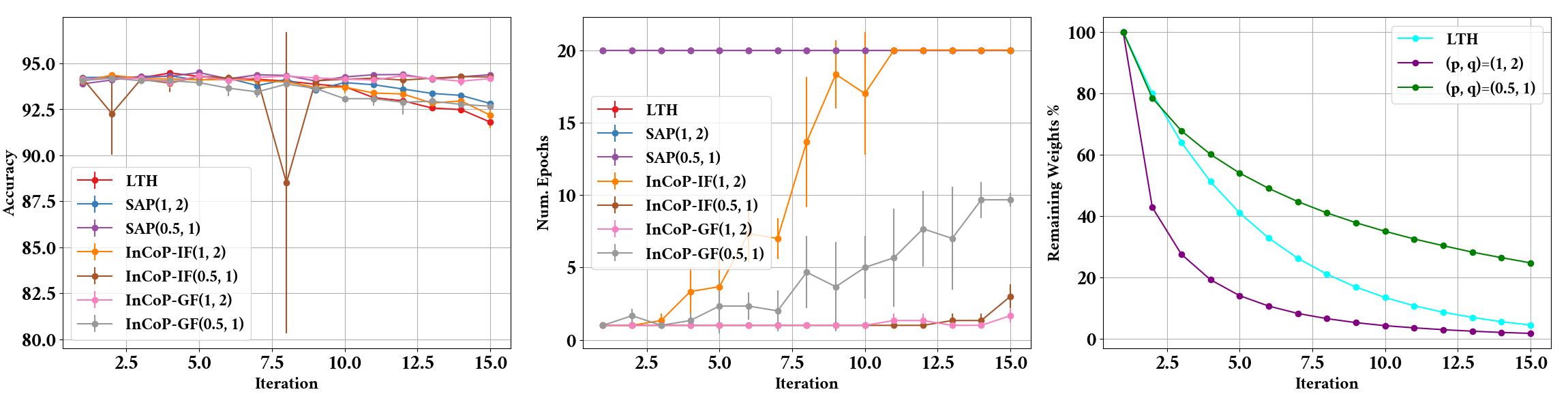}
  \caption{VGG16 - MNIST~(top), FashionMNIST~(middle), CIFAR10~(bottom). As for Fig.~\ref{fig:resnet} purple line indicates $(p, q) = (1, 2)$, green is  $(p, q)=(0.5, 1)$, and LTH is cyan line. }
\label{fig:vgg}
\end{figure*}

As illustrated in Fig.~\ref{fig:resnet} and \ref{fig:vgg}, the accuracy of the network remains relatively stable across different methods, with minor variations due to the inherent stochasticity of the process. However, the data in the middle graph clearly indicates that our proposed approach achieves comparable levels of accuracy with a notably reduced number of training epochs.

It should be noted that the remaining weights in SAP and \ourmethod{} are identical for the same $(p, q)$.
This is due to the fact that \ourmethod{} utilize the same approach as SAP to determine the number of nodes to prune.
With that said, the concept of information-based stopping criterion is independent of the pruning method and can be employed in all other pruning techniques.
Figs.~\ref{fig:resnet} and \ref{fig:vgg} indicate that \ourmethod{}-IF is more prone to aggressive pruning, as it requires more training epochs in later iterations.
This is due to the fact that connectivity is computed for two consecutive layers and as we prune more neurons in each layer, the connectivity drops in $\mathcal{O}(M_l^2)$.
In contrast to \ourmethod{}-IF, this is not the case with \ourmethod{}-GF, as it demonstrates consistent performance regardless of the pruning ratio applied.

A detailed \textbf{ablation study} on hyperparameters, specifically focusing on $\epsilon$ is provided in the supplementary material (SM). A key finding in these results is that increasing the threshold $\epsilon$ in \ourmethod{}, the method tends to exhibit characteristics closer to SAP, behaving more conservatively. Conversely, decreasing $\epsilon$ may result in suboptimal performance, as the algorithm may prematurely stop training before the network achieves sufficient performance gains for the subsequent iteration.

We have also presented experimental results on deeper networks such as ResNet-56 in SM. The results, however, are aligned with the experiments discussed earlier in this paper, demonstrating consistency across different network architectures and datasets.


\subsection{Time Efficiency Analysis}
In Figs.~\ref{fig:resnet} and \ref{fig:vgg} we demonstrate our proposed method exhibits enhanced efficiently in terms of training epochs across diverse datasets and network architectures.
This improvement in the number of training epochs is crucial for optimizing the retraining time of the IMP algorithm, which is the most time-consuming step in the overall process.
This is equivalent to having an adaptive \textit{early stopping} criterion that does not only rely solely on accuracy, but rather works with a more informed metric such as IF and GF.
However, it is important to note that achieving this efficiency involves additional computational costs associated with calculating the GF and IF.


The computation of connectivity as shown in equation \eqref{eq:def-connectivity} has a time complexity of $\mathcal{O}(N)$. Specifically, this connectivity calculation occurs in every training step of the IMP algorithm, resulting in an additional overall overhead of $\mathcal{O}(T \cdot E \cdot N \cdot L \cdot M_l \cdot M_{l+1}$). Fortunately, the value of $E$ for \ourmethod{}-IF is very minimal and can be disregarded, while $T, L, M_l$, and $M_{l+1}$ are constants that can be eliminated from consideration. Similarly, \ourmethod{}-GF introduces computational overhead compared to the original SAP. In each iteration, we need to compute the GF for each layer, that is $\mathcal{O}(N \cdot L \cdot M_l^2)$, which again considering the fact that $M_l, L$ are constant becomes $\mathcal{O}(N)$.

To evaluate the time complexity of \ourmethod{}, we conducted two sets of experiments for a) comparing the execution times of computing GF and IF in isolation and b) comparing the overall computational time of our method with LTH and SAP.
In the first experiment, the computation times were recorded as $122.0\pm 7.9$ seconds for IF calculations and $61.3 \pm 1.3$ seconds for GF calculations in the experiment using ResNet18 on the FashionMNIST dataset.

Figure \ref{fig:time_comparison} presents the second experiment, a time comparison among various methods for one of the scenarios detailed in Performance Comparison.
The experiment involved measuring the overall execution time for running SAP, LTH, \ourmethod{}-IF, and \ourmethod{}-GF on FashionMNIST dataset using ResNet18 for 15 pruning iterations and with identical hyperparameters as described in Performance Comparison. As Table \ref{fig:time_comparison} depicts, incorporating GF and IF not only did not increase the execution time of the pruning algorithm, but also notably reduced it. The efficiency of \ourmethod{} in determining the optimal stopping point for the retraining step is the reason for this reduction.


\begin{figure}[hb!]
  \centering
  \includegraphics[scale=0.35]{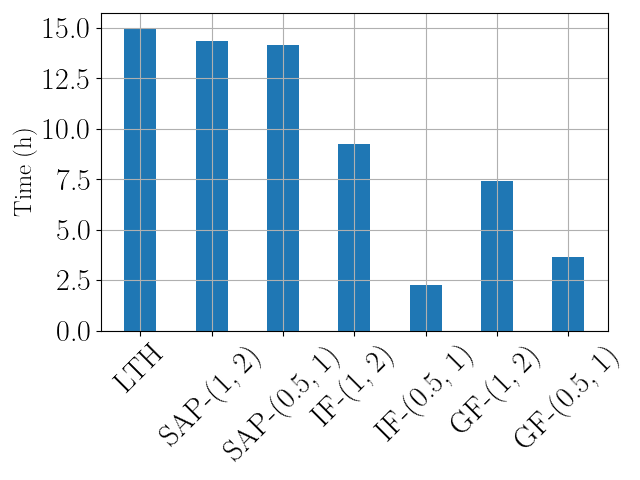}
  \caption[Time Comparison]{Comparison of execution times for a single IMP experiment across different pruning methods.}
  \label{fig:time_comparison}
\end{figure}

\section{Related Work}

{\bf Lottery Ticket Hypothesis} LTH was introduced in \cite{frankle-2018-lotter-ticket-hypot}, and further improved by the concept of \textit{rewinding}, where the network's weights are reset to a network trained for $k$ epochs~\cite{frankle-2019-stabil-lotter} instead of resetting to initial values.
Recent studies \cite{orseau2020logarithmic,malach2020proving} have provided mathematical analysis and explanation for the success of this hypothesis.
They also proved a stronger version of this hypothesis and showed that inside every sufficiently over-parameterized neural network, there exists a subnetwork with roughly the same accuracy.

{\bf Pruning vs. training step} While some studies have focused on acquiring analytical insights into LTH, others have pursued various approaches to optimize it~\cite{frankle-2019-stabil-lotter,you2019drawing,morcos2019one,zullich2021speeding}.
IMPs consist of two crucial steps: \textit{pruning} and \textit{training} the network back to its original accuracy, with different strategies proposed to address these key steps.

In \cite{evci2020rigging}, a variation of LTH is proposed focusing on the pruning step, where instead of pruning a fixed number of parameters, a constant number of floating point computations (FLOPs) is maintained constant throughout the process. Similarly, an adaptive pruning algorithm was proposed in \cite{diao-2023-prunin-deep}, where instead of a fixed pruning ratio, a dynamic one is utilized that changes according to the compressibility of the network.

To optimize the training step in IMPs, \cite{wang2020picking} demonstrated that by leveraging IF one can enhance pruning at the initial stage of the process, leading to time savings during training. The study by \cite{evci-2020-gradien-flow} revealed that pruning neural networks typically degrades GF. They sought to uncover the mechanism by which LTH compensates for this loss in GF, ultimately enhancing network generalization.
They determined that the training step is the pivotal element of the process, driving the enhancement in the performance of LTH. In a similar work \cite{tessera2021keep} delved into optimizing sparse networks from various aspects, including initialization, regularization, and architecture choice. They also introduced EGF as a practical proxy for evaluating network performance.


{\bf Pruning criterion vs. stopping criterion} The pruning step in IMP algorithm comprises two components: the \textit{pruning criterion} and the \textit{stopping criterion}. The pruning criterion determines which connection or nodes in the network have to be eliminated, while the stopping criterion determines the point at which the retraining process stops. The predominant pruning criterion in use is weights magnitudes \cite{zhou2018online}. Magnitude-based pruning algorithms are predicated on the notion that weights with lower magnitudes have minimal impact on the network's output. However, researchers have challenged this perspective and proposed alternative approaches, such as sensitivity-based pruning \cite{lee-2019-snip,hu2016network,lin2018accelerating,haider-2020-compr-onlin}.

In sensitivity-based pruning, the importance of each node is assessed based on sensitivity criteria, \textit{individually}, and then an optimization problem is solved to maximize the performance of the network. This problem is known to be NP-hard and cannot be directly solved, only approximate solutions exist. The high complexity and sub-optimality of these methods compared to magnitude-based techniques pose challenges.

In addition to magnitude and sensitivity, researchers have explored alternative pruning criteria such as the IF. 
For example, in \cite{ganesh2022slimming}, mutual information was utilized to analyze IF within sparse networks and guide the pruning process based on the \textit{connectivity} between network layers. Connectivity-based network training has also been showing promising performance in dealing with catastrophic forgetting in continual learning \cite{andle-2022-theor-under}. 
IF and its relation to the generalization ability of neural networks has been previously investigated in \cite{jin-2020-how-does}. The primary limitation of information-based methods lies in their significant computational overhead, rendering them less appealing for an iterative process.

\section{Discussion and Conclusion}
In this study, we introduce a novel stopping criterion for IMPs that leverages connectivity and GF to improve the training time of the algorithm.
IMP methods have demonstrated outstanding performances by effectively reducing the parameters of the network to a fraction of its original size.
However, the iterative nature of IMP results is slow and resource-intensive operations.
By utilizing a novel stopping criterion, we effectively reduce the training epochs required for the retraining phase of IMP, thereby eliminating redundant training epochs in the intermediate steps.

Our experimental findings clearly illustrate the superior efficiency of our approach when compared to existing methods.
Our algorithm stops the training even if the accuracy of the network has not reached the original accuracy of the dense network during intermediate steps.
Stopping the training at the right time is a crucial decision that cannot be made without the guidance of IF and GF of the network. Our experiments demonstrate that we achieve the same final accuracy as SAP at the end of the process, with only $13\%$ of execution time for the case of \ourmethod{}-IF with $(p, q)=(0.5, 1)$.
This step is crucial for advancing IMP algorithms, especially considering that an advanced method like SAP has already successfully achieved optimal network compression in minimal time.

Despite the computational intensity of GF and IF calculations, our experiments have demonstrated the validity of using these metrics as a stopping criterion. This decision is justified by the substantial speed improvement that our algorithm offers compared to a standard IMP method.

As we move forward, there are several avenues for further investigation in this problem. Some potential directions for future work include:
\begin{enumerate*}[label=(\roman*)]
    \item Leveraging the insights from connectivity and GF to optimize the \textit{pruning} step allows us to strategically prune the network while minimizing the damage to the information and GF within the network.

    \item Exploring alternative metrics such as cosine similarity, rather than the $L_2$ norm, to assess proximity to the optimal network. This approach will aim to identify a more refined method for retraining the network.

    \item Exploring the distribution shift effect, where the network is pruned according to one dataset, and trained using another one.
This idea can be expanded into the field of knowledge transfer and knowledge distillation as well.
\end{enumerate*}

\section*{Acknowledgments}
This work has been partially supported by NSF CAREER-CCF 2451457 and the findings are those of the authors only and do not represent any position of these funding bodies.

\input{SM_included}

\bibliographystyle{splncs04}
\bibliography{main}

\end{document}

%% file: SM_included.tex
\def\BX{\mathbf{X}}
\def\CX{\mathcal{X}}
\def\bx{\mathbf{x}}
\def\bX{\mathbf{X}}
\def\CT{\mathcal{T}}
\def\bbE{\mathbb{E}}

\section{Supplementary Material}
\subsection{Proofs}
In this section, we present our mathematical proofs illustrating why the flow-based stopping criterion is more effective than conventional accuracy-based algorithms. To support this assertion, we establish that when the network's connectivity remains in proximity to the optimal network, the performance differences also remain within the vicinity of the optimal levels.

\textit{{\bf Analytical Hypothesis:} The network's performance remains within a small vicinity of the optimal network if the IF difference between dense network and sparse network i.e. \\
$|\Delta_{w_e}-\Delta_{w^*}|$ is within small $\epsilon\geq 0$.  }

Suppose $\bX$ and $Y$ have distribution $\mathcal{D}$.
Training a DNN $F^{(L)}\in \mathcal{F}$ is performed by minimizing a loss function (empirical risk) that decreases with the correlation between the weighted combination of the networks and the label $Y$ as $\mathbb{E}_{(\mathbf{X},Y)\sim \mathcal{D}}\big\{Y\cdot\big(b+\sum\limits_{F\in\mathcal{F}}w_{F}\cdot F^{(L)}(\BX)\big)\big\}$. 
We remove offset $b$ without loss of generality. Define $\ell(w):=-\sum\limits_{F\in\mathcal{F}}w_{F}\cdot F^{(L)}(\mathbf{X})$,
and let $w^*$ be the set of optimal parameters and $ w^*:= \arg\min_{w} \mathbb{E}_{(\mathbf{X},Y)\sim D}\left\{ Y\cdot\left(\ell({w})\right)\right\}$. We define the weights of $F^{(L)}$ at training epoch $e$ by $w_e$.  
The difference between loss functions at $w^*$ and $w_e$ is given by  
\begin{equation}\label{eq:performances-bound}
  \begin{split}
|Acc^*-Acc_e|:=\big|\Ex{(\bX,Y)\sim \mathcal{D}}{Y\cdot\ell(w^*)}- \\
  \Ex{(\bX,Y)\sim \mathcal{D}}{Y\cdot\ell(w_e)}\big|.
  \end{split}
\end{equation}
Inspired by Lemma~3 (inverse Jensen inequality) from \cite{wunder2021reverse}, for constants $a_1$ and $a_2$,
 set $a=a_2/|a_1|>0$, therefore we bound (\ref{eq:performances-bound}) as following 
\begin{equation}\label{eq:performances-bound1}
|Acc^*-Acc_e| \geq a\;\Ex{(\bX,Y)\sim \mathcal{D}}{Y\cdot|\ell(w_e)-\ell(w^*)|}
\end{equation}
Using the second-order Taylor approximation of $\loss$ around $w^*$ (optimal point), we have 
  \begin{equation}
    \label{eq:taylor-expantion}
  \loss(w_e) - \loss(w^*) \approx \frac{1}{2} (w_e - w^*)^T \nabla^2
  \loss(w^*)(w_e - w^*), 
  \end{equation}
because $\nabla \ell(w^*)=0$.

Let $\lambda_{\min}$ be the maximum eigenvalue of $\nabla^2 \loss(w^*)$. Combining (\ref{eq:performances-bound1}) and (\ref{eq:taylor-expantion}), we have 
\begin{align}\label{eq:1}
|Acc_e-Acc^*| \geq \frac{a}{2}\; \Ex{(\bX, Y) \sim \mathcal{D}}{ Y.\lambda_{\min} \|
    w_e - w^* \|^2 },
\end{align}
and $a>0$.
Let $\Delta_e=[\Delta_{w_e}(1), \ldots, \Delta_{w_e}(l),\ldots, \Delta_{w_e}(L)]$ be the connectivity vector of all layers at epoch $e$ with weights $w_e$. In addition, note that  $f^{(l+1)}=\sigma^{(l)}(ww^{(l)} f^{(l)}+b^{(l)})$,
for simplicity ignore the offset $b^{(l)}$. Suppose $\sigma(.)$ is a Lipschitz continuous function. Set
$\bar{\sigma}_{w^*} (t)=t \; \sigma(w^* t)$, 
\begin{equation}\label{eq:element-wise-conn-diff}
f^{(l)}_{w_e}\; f_{w_e}^{(l+1)}-f^{(l)}_{w^*}
\; f_{w^*}^{(l+1)}= \bar{\sigma}_{w_e} (f^{(l)}_{w_e})-\bar{\sigma}_{w^*}(f^{(l)}_{w^*}).
\end{equation}
Next, recall Lemmas~1 and 2 from SM, we find the upper bound as follows
  $$|\bar{\sigma}_{w_e}(t)-\bar{\sigma}_{w^*}(s)|\leq C^{(1)}_\sigma\Vert{w_e}-w^*\Vert^2+C^{(2)}_\sigma \Vert t-s\Vert^2,$$ 
where $C^{(1)}_\sigma,C^{(2)}_\sigma$ are constants. Combined with (\ref{eq:element-wise-conn-diff}) implies that  
\begin{equation}
    \vert\Delta_{w_e}(l)-\Delta_{w^*}(l)\vert_{ij}\leq\mathbb{E}_{(\BX,Y)\sim \mathcal{D}}\left[ \big\vert C^{(1)}_\sigma \Vert{w_e}-w^*\Vert^2+
    C^{(2)}_\sigma\Vert f^{(l)}_{w_e}-f^{(l)}_{w^*}\big\Vert^2  \vert \right]_{ij}, 
  \end{equation}
Since the difference between filters in layer $l$ is bounded:
  \begin{equation}\label{eq-01}
     \sum_{l=1}^L \vert\Delta_{w_e}(l)-\Delta_{w^*}(l)\vert\leq \mathbb{E}_{(\BX,Y)\sim \mathcal{D}}[ \overline{C}_\sigma \Vert{w_e}-w^*\Vert^2+L C_\sigma], 
  \end{equation}
where $\overline{C}_\sigma$ and $C_\sigma$ are constants and $L$ is total number of layers in the network. 
Assume $  \sum_{l=1}^L \vert\Delta_{w_e}(l)-\Delta_{w^*}(l)\vert\geq \epsilon$ for given $\epsilon\geq L C_\sigma$, then 

  \begin{equation}\label{eq-02}
\mathbb{E}_{(\BX,Y)\sim D}[ \overline{C}_\sigma \Vert{w_e}-w^*\Vert^2]\geq \epsilon-L C_\sigma. 
  \end{equation}
We claim that $\exists\; \widetilde{C}$ such that $\widetilde{C}\geq \overline{C}_\sigma$ and
\begin{equation}
  \mathbb{E}_{(\BX,Y)\sim D}[\widetilde{C}\Vert{w_e} - w^*\Vert^2]\leq  
  \frac{a}{2}  \; \Ex{(\bX, Y) \sim \mathcal{D}}{ Y.\lambda_{\min} \| 
    w_e - w^* \|^2 },   
\end{equation}
Combining (\ref{eq:1}) and (\ref{eq-02}), we conclude that $|Acc_e-Acc^*|\geq \epsilon-L C_\sigma$. This shows that as the $\Delta_{w_e}$ gets closer to $\Delta_{w^*}$ another word $\epsilon\rightarrow 0$, then the accuracy difference between training at epoch $e$ ($Acc_e$) and best performance ($Acc^*$) gets closer to zero. We infer from this analysis that connectivity $\Delta$ at epoch $e$ provides clear information on stopping the training and propose Algorithm~\ref{algo.1}: 

\begin{lemma} \label{lem:lipschits-1}
Assume activation function $\sigma_{w}(.)$ is Lipschitz continuous function and bounded. Then function  
$\bar{\sigma}_{w} (t)=t \; \sigma(w\; t)$ is also Lipschitz continuous.
\end{lemma}
\begin{proof}
Recall Lipschitz continuous property for activation function $\sigma_w$, then there exists a constant $C_\sigma$ such that 
\begin{equation}\label{LP}
|\sigma (t)-\sigma(s)|\leq C^{(1)}_\sigma \|t-s\|^2.
\end{equation}
Given the function $\bar{\sigma}_{w} (t)=t \; \sigma(w\; t)$, we have 
  \begin{align}\label{eq:2}
 | \bar{\sigma}_w (t) - \bar{\sigma}_w(s) | & = | t \sigma(wt) - s \sigma(ws) | \\
&= | t \sigma(wt) - t \sigma(ws) + 
 t \sigma(ws) - s \sigma(ws) |
 \end{align}
 Applying triangle inequality we bound (\ref{eq:2}) by 
 \begin{align}
  | \bar{\sigma}_w (t) - \bar{\sigma}_w(s) | &  \leq \|t\|^2 | \sigma(wt) - \sigma(ws) +
 |\sigma(ws)| \|t - s\|^2
 \end{align}
 Now because $\sigma$ is Lipschitz continuous (\ref{LP}), $\sigma$ is bounded by $C^{(2)}_\sigma$, and $\|t\|^2\leq C_t$,  we have 
 \begin{align}
  | \bar{\sigma}_w (t) - \bar{\sigma}_w(s) | & \leq C_t C^{(1)}_{\sigma} \| wt - ws \|^2 +
 |\sigma(ws)| \| t - s\|^2 \\
 & \leq (C_t C^{(1)}_{\sigma} \|w\|^2 + | C^{(2)}_\sigma | ) \|t - s\|^2.
 \end{align}
 This concludes that for constant $\overline{C}_\sigma:=(C_t C^{(1)}_{\sigma} \|w\|^2 + | C^{(2)}_\sigma | )$ and $\|w\|^2\leq C_w$, we can write 
 \begin{equation}
| \bar{\sigma}_w (t) - \bar{\sigma}_w(s) |\leq \overline{C}_\sigma \|t - s\|^2, 
 \end{equation}
 which proves $\bar{\sigma}$ is Lipschitz continuous.
\end{proof}

\begin{lemma} \label{lem:lipschits-2}
If $\sigma_{w}(.)$ is Lipschitz continuous, we can upper bound
  $|\bar{\sigma}_{w_e}(t)-\bar{\sigma}_{w^*}(s)|$ by 
\begin{align}\label{main-bound}
 | \bar{\sigma}_{w_e}(t) - \bar{\sigma}_{w^*}(s) | \leq C^{(1)} \| t - s\|^2 + C^{(2)} \|w_e-w^*\|^2, 
\end{align}
where $C^{(1)}$ and $C^{(2)}$ are constants. 
\begin{proof}
Given $\bar{\sigma}_w(t)=t \;\sigma(w\;t)$, we have 
\begin{align} \label{eq:ub-lipschitz}
 | \bar{\sigma}_{w_e}(t) - \bar{\sigma}_{w^*}(s) | & = | t \sigma(w_e\;t) - s
        \sigma(w^*\;s) | \nonumber\\
    & = | t \sigma(w_e\; t) - s \sigma (w_e\; t) + \nonumber\\
    & \;\;\;\;\;  s \sigma (w_e\; t) - s \sigma(w^*\; s) |
\end{align}
Apply triangle inequality to (\ref{eq:ub-lipschitz}), we have 
\begin{align}\label{bound.1}
  | \bar{\sigma}_{w_e}(t) - \bar{\sigma}_{w^*}(s) | \leq |\sigma(w_e\; t)| \| t - s\|^2 + \nonumber\\
  \|s\|^2 |\sigma(w_e\; t) - \sigma(w^*\; s) |.  
\end{align}
Because $\|s\|^2\leq C_s$ and $\sigma$ is bounded by $C^{(1)}_\sigma$ for constants $C_s$ and $C^{(1)}_\sigma$, we bound (\ref{bound.1}) as follows
\begin{align}
 | \bar{\sigma}_{w_e}(t) - \bar{\sigma}_{w^*}(s) | \leq C_\sigma \| t - s\|^2 + C_s |\sigma(w_e\; t) - \sigma(w^*\; s) |.
\end{align}
In addition by using Lipschitz continuous property of $\sigma$, for constant $C^{(2)}_\sigma$ we can write 
\begin{align}\label{bound2}
  | \bar{\sigma}_{w_e}(t) - \bar{\sigma}_{w^*}(s) | \leq C^{(1)}_\sigma \| t - s\|^2 + \nonumber \\
  C_s C^{(2)}_\sigma \|w_e\; t - w^*\; s\|^2.
\end{align}
Also, we have 
\begin{align}\label{bound3}
\|w_e\; t - w^*\; s\|^2 &= \|w_e t -w_e s +w_e s -w^* s\|^2\\
& \leq \|w_e\|^2\|t-s\|^2+\|s\|^2\|w_e-w^*\|^2\\
&\leq C_e \|t-s\|^2 + C_s \|w_e-w^*\|^2.
\end{align}
The combination of  (\ref{bound2}) and (\ref{bound3}) implies 
\begin{equation}\label{bound4}
  | \bar{\sigma}_{w_e}(t) - \bar{\sigma}_{w^*}(s) | \leq C^{(1)}_\sigma \| t - s\|^2 + 
  C_s C^{(2)}_\sigma \left( C_e \|t-s\|^2 + C_s \|w_e-w^*\|^2\right). 
\end{equation}
We simplify (\ref{bound4}) as 
\begin{align}\label{bound5}
 | \bar{\sigma}_{w_e}(t) - \bar{\sigma}_{w^*}(s) | \leq C^{(1)} \| t - s\|^2 + C^{(2)} \|w_e-w^*\|^2, 
\end{align}
where $C^{(1)}=C^{(1)}_\sigma+C_s C^{(2)}_\sigma C_e$ and $C^{(2)}=C^2_s C^{(2)}_\sigma $. This concludes our claim in (\ref{main-bound}). 
  \end{proof}
\end{lemma}
{\bf Remark:} {\it (Lemma~2 proof using Lemma~1)} 
Based on Lipschitz property of $\sigma$, and Lemma~1, there exists a constant $\overline{C}_\sigma$ such that
\begin{equation}\label{eq.5}
\begin{aligned}
|\bar{\sigma}_{w_e}(t)-\bar{\sigma}_{w^*}(s)|&= |\bar{\sigma}_{w_e}(t)-\bar{\sigma}_{w_e}(s)+ \bar{\sigma}_{w_e}(s)-\bar{\sigma}_{w^*}(s)|\\
&\leq \overline{C}_\sigma \|t-s\|^2 +|\bar{\sigma}_{w_e}(s)-\bar{\sigma}_{w^*}(s)|.
\end{aligned}\end{equation}
Going back to the definition of $\bar\sigma$: 
\begin{align}
|\bar{\sigma}_{w_e}(t)-\bar{\sigma}_{w^*}(s)| &\leq \overline{C}_\sigma \|t-s\|^2 + |s\;\sigma(w_e\; s)-s\sigma(w^* s)|\\ 
& \leq \overline{C}_\sigma \|t-s\|^2 + C^2_s C_\sigma \|w_e-w^*\|^2, 
\end{align}
which again implies (\ref{main-bound}) with $C^{(1)}=\overline{C}_\sigma$ and $C^{(2)}=C^2_s C_\sigma$. 


\begin{lemma}\label{lemma.1}
(Lemma 3 - Reverse Jensen Inequality~\cite{wunder2021reverse}) Let $f:\mathbb{R}_0^+\mapsto \mathbb{R}_0^+$ be convex, increasing, and $f(0)=0$. Then for any random variable $Z\geq 0$, 
\begin{align}
f(a_1 \mathbb{E}(Z))\geq a_2 \mathbb{E}(f(Z)), \;\;\hbox{where $a_1$ is a function of $a_2>0$.}
\end{align}
\end{lemma}

\subsection{Experimental Results}

Our experimentation also extends to deeper networks such as ResNet50 as well, in which due to the space limitations we could not include in the experiment section.
In this experiment, we trained ResNet18 with identical parameters as detailed in the Experimental Results section. The results displayed in Figure \ref{fig:cifar100} indicate that \ourmethod{}-IF outperformed \ourmethod{}-GF, showcasing superior performance while necessitating a reduced number of training epochs. This outcome is consistent with the results observed in the body of the paper, showing that our method works across various sizes of network and dataset.

\begin{figure}
  \centering
  \includegraphics[scale=.20]{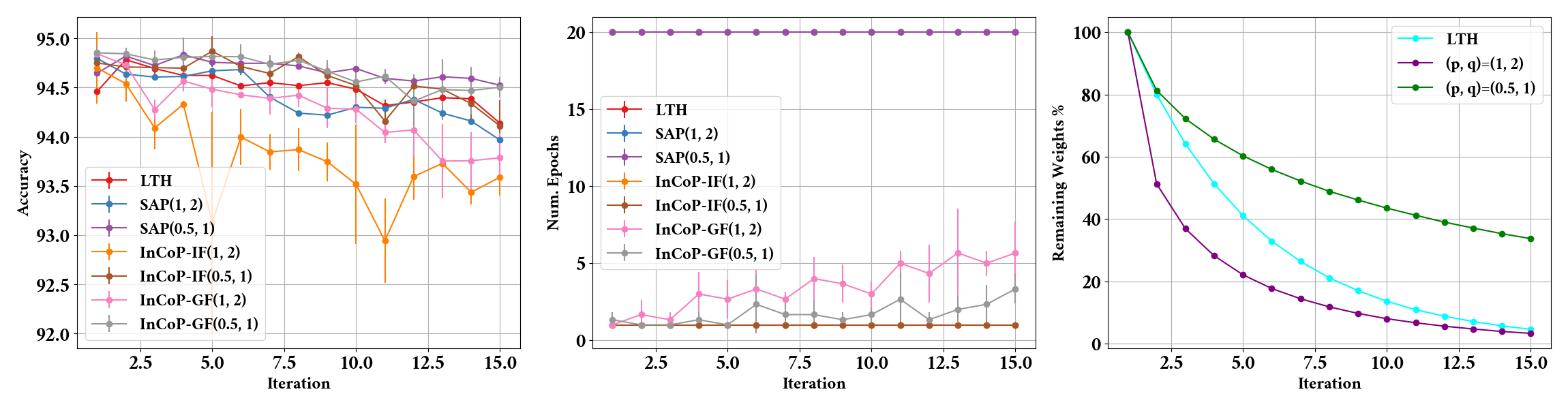}
\caption{ResNet50 with FashionMNIST. As for Fig.~\ref{fig:resnet} purple line $\rightarrow (p, q) = (1, 2)$, green $\rightarrow (p, q)=(0.5, 1)$, and cyan $\rightarrow$ LTH.}
\label{fig:cifar100}
\end{figure}

\paragraph{Hyperparameters Study}
In our ablation study, we experimented with various amounts of $\epsilon$ for our two proposed methods, CIAP and GIAP.

\begin{figure}
  \centering
  \includegraphics[scale=.20]{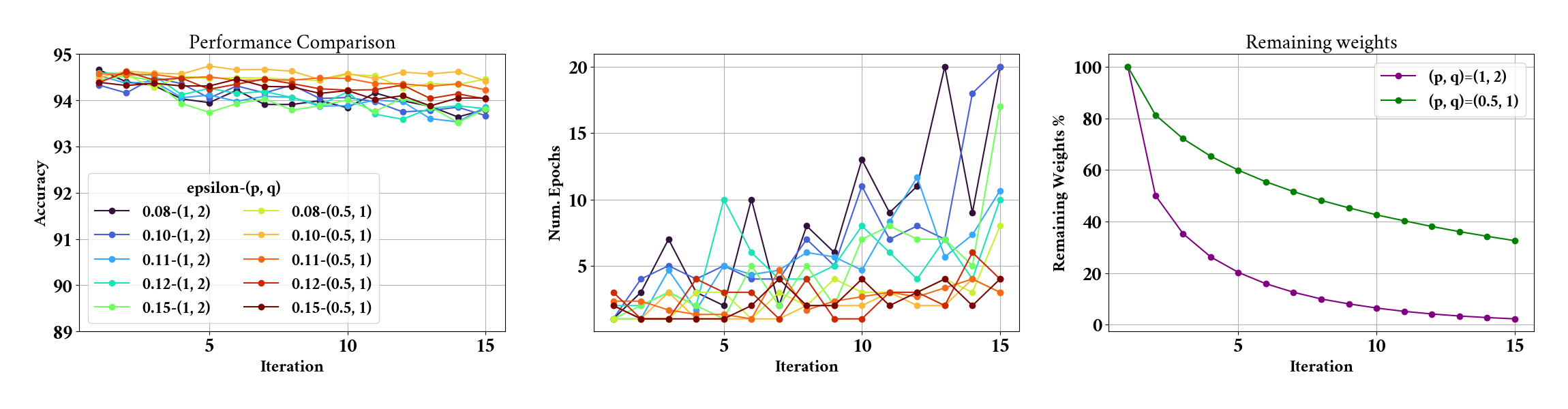}
  \includegraphics[scale=.20]{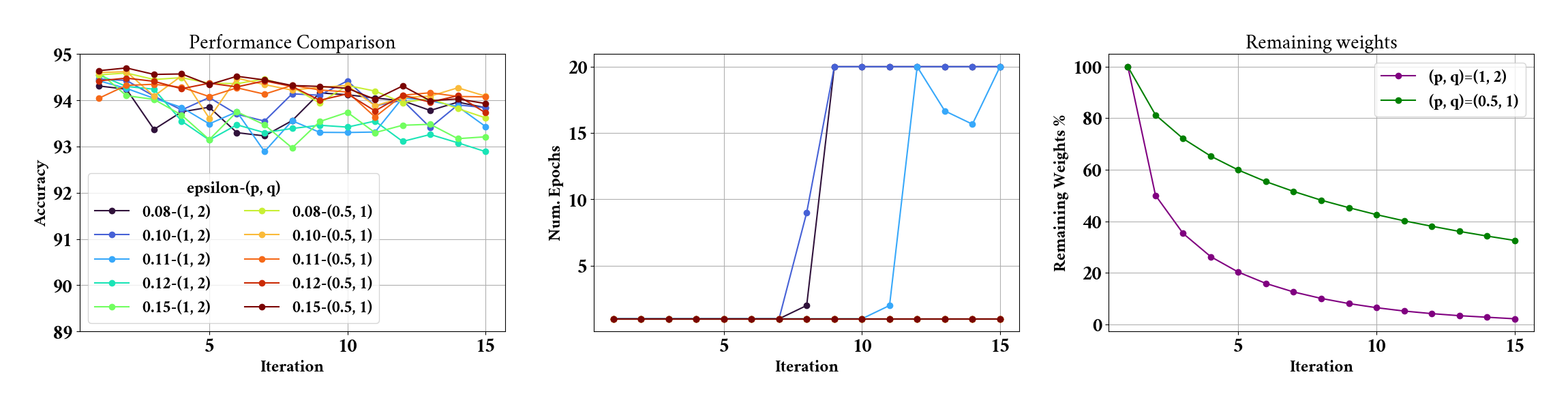}
\caption{ResNet18 on FashionMNIST: Comparing \ourmethod{}-IF (top), and \ourmethod{}-GF (bottom) for various amounts of epsilon for two pairs of $(p, q)$}
\label{fig:ablation}
\end{figure}

Increasing $\epsilon$ allows for greater distance between the optimal network and the sparse network, resulting in a less optimal network solution. However, this trade-off can lead to quicker convergence with fewer training epochs.
Moreover, in this experiment, the sensitivity of \ourmethod{}-IF to higher pruning ratios becomes evident. As the iterations progress and more neurons are pruned from consecutive layers, there is a sudden need for the maximum number of training epochs to counterbalance the performance degradation resulting from the pruning process.

%% file: main.bbl
\begin{thebibliography}{10}
\providecommand{\url}[1]{\texttt{#1}}
\providecommand{\urlprefix}{URL }
\providecommand{\doi}[1]{https://doi.org/#1}

\bibitem{ahmad2019can}
Ahmad, S., Scheinkman, L.: How can we be so dense? the benefits of using highly
  sparse representations. arXiv preprint arXiv:1903.11257  (2019)

\bibitem{andle-2022-theor-under}
Andle, J., Sekeh, S.Y.: Theoretical Understanding of the Information Flow on
  Continual Learning Performance, pp. 86--101. Lecture Notes in Computer
  Science, Springer Nature Switzerland (2022).
  \doi{10.1007/978-3-031-19775-8_6},
  \url{http://dx.doi.org/10.1007/978-3-031-19775-8_6}

\bibitem{chen2021only}
Chen, T., Ji, B., Ding, T., Fang, B., Wang, G., Zhu, Z., Liang, L., Shi, Y.,
  Yi, S., Tu, X.: Only train once: A one-shot neural network training and
  pruning framework. Advances in Neural Information Processing Systems
  \textbf{34},  19637--19651 (2021)

\bibitem{deng2012mnist}
Deng, L.: The mnist database of handwritten digit images for machine learning
  research. IEEE Signal Processing Magazine  \textbf{29}(6),  141--142 (2012)

\bibitem{dettmers2019sparse}
Dettmers, T., Zettlemoyer, L.: Sparse networks from scratch: Faster training
  without losing performance. arXiv preprint arXiv:1907.04840  (2019)

\bibitem{diao-2023-prunin-deep}
Diao, E., Wang, G., Zhan, J., Yang, Y., Ding, J., Tarokh, V.: Pruning deep
  neural networks from a sparsity perspective. CoRR  (2023),
  \url{http://arxiv.org/abs/2302.05601v3}

\bibitem{evci2020rigging}
Evci, U., Gale, T., Menick, J., Castro, P.S., Elsen, E.: Rigging the lottery:
  Making all tickets winners. In: International conference on machine learning.
  pp. 2943--2952. PMLR (2020)

\bibitem{evci-2020-gradien-flow}
Evci, U., Ioannou, Y.A., Keskin, C., Dauphin, Y.: Gradient flow in sparse
  neural networks and how lottery tickets win. CoRR  (2020),
  \url{http://arxiv.org/abs/2010.03533v2}

\bibitem{frankle-2018-lotter-ticket-hypot}
Frankle, J., Carbin, M.: The lottery ticket hypothesis: Finding sparse,
  trainable neural networks. CoRR  (2018),
  \url{http://arxiv.org/abs/1803.03635v5}

\bibitem{frankle-2019-linear-mode}
Frankle, J., Dziugaite, G.K., Roy, D.M., Carbin, M.: Linear mode connectivity
  and the lottery ticket hypothesis. CoRR  (2019),
  \url{http://arxiv.org/abs/1912.05671v4}

\bibitem{frankle-2019-stabil-lotter}
Frankle, J., Dziugaite, G.K., Roy, D.M., Carbin, M.: Stabilizing the lottery
  ticket hypothesis. CoRR  (2019), \url{http://arxiv.org/abs/1903.01611v3}

\bibitem{ganesh-2021-slimm-neural}
Ganesh, M.R., Blanchard, D., Corso, J.J., Sekeh, S.Y.: Slimming neural networks
  using adaptive connectivity scores (2021),
  \url{http://arxiv.org/abs/2006.12463}

\bibitem{ganesh2022slimming}
Ganesh, M.R., Blanchard, D., Corso, J.J., Sekeh, S.Y.: Slimming neural networks
  using adaptive connectivity scores. IEEE Transactions on Neural Networks and
  Learning Systems  (2022)

\bibitem{ganesh2024slimming}
Ganesh, M.R., Blanchard, D., Corso, J.J., Sekeh, S.Y.: Slimming neural networks
  using adaptive connectivity scores. IEEE transactions on neural networks and
  learning systems  \textbf{35}(3),  3794--3808 (2024)

\bibitem{ganesh2021mint}
Ganesh, M.R., Corso, J.J., Sekeh, S.Y.: Mint: Deep network compression via
  mutual information-based neuron trimming. In: 2020 25th International
  Conference on Pattern Recognition (ICPR). pp. 8251--8258. IEEE Computer
  Society (2021)

\bibitem{haider-2020-compr-onlin}
Haider, M.U., Taj, M.: Comprehensive online network pruning via learnable
  scaling factors. CoRR  (2020), \url{http://arxiv.org/abs/2010.02623v1}

\bibitem{han-2015-learn-both}
Han, S., Pool, J., Tran, J., Dally, W.J.: Learning both weights and connections
  for efficient neural networks. CoRR  (2015),
  \url{http://arxiv.org/abs/1506.02626v3}

\bibitem{he2016deep}
He, K., Zhang, X., Ren, S., Sun, J.: Deep residual learning for image
  recognition. In: Proceedings of the IEEE conference on computer vision and
  pattern recognition. pp. 770--778 (2016)

\bibitem{hinton2012deep}
Hinton, G., Deng, L., Yu, D., Dahl, G.E., Mohamed, A.r., Jaitly, N., Senior,
  A., Vanhoucke, V., Nguyen, P., Sainath, T.N., et~al.: Deep neural networks
  for acoustic modeling in speech recognition: The shared views of four
  research groups. IEEE Signal processing magazine  \textbf{29}(6),  82--97
  (2012)

\bibitem{hu2016network}
Hu, H., Peng, R., Tai, Y.W., Tang, C.K.: Network trimming: A data-driven neuron
  pruning approach towards efficient deep architectures. arXiv preprint
  arXiv:1607.03250  (2016)

\bibitem{ioffe2015batch}
Ioffe, S., Szegedy, C.: Batch normalization: Accelerating deep network training
  by reducing internal covariate shift. In: International conference on machine
  learning. pp. 448--456. pmlr (2015)

\bibitem{jin-2020-how-does}
Jin, G., Yi, X., Zhang, L., Zhang, L., Schewe, S., Huang, X.: How does weight
  correlation affect the generalisation ability of deep neural networks. CoRR
  (2020), \url{http://arxiv.org/abs/2010.05983v3}

\bibitem{krizhevsky2010convolutional}
Krizhevsky, A., Hinton, G.: Convolutional deep belief networks on cifar-10.
  Unpublished manuscript  \textbf{40}(7), ~1--9 (2010)

\bibitem{krizhevsky2012imagenet}
Krizhevsky, A., Sutskever, I., Hinton, G.E.: Imagenet classification with deep
  convolutional neural networks. Advances in neural information processing
  systems  \textbf{25} (2012)

\bibitem{lecun1989optimal}
LeCun, Y., Denker, J., Solla, S.: Optimal brain damage. Advances in neural
  information processing systems  \textbf{2} (1989)

\bibitem{lee-2019-snip}
Lee, N., Ajanthan, T., Torr, P.H.S.: {SNIP}: Single-shot network pruning based
  on connection sensitivity (2019), \url{http://arxiv.org/abs/1810.02340}

\bibitem{lin2018accelerating}
Lin, S., Ji, R., Li, Y., Wu, Y., Huang, F., Zhang, B.: Accelerating
  convolutional networks via global \& dynamic filter pruning. In: IJCAI.
  vol.~2, p.~8. Stockholm (2018)

\bibitem{luo2017thinet}
Luo, J.H., Wu, J., Lin, W.: Thinet: A filter level pruning method for deep
  neural network compression. In: Proceedings of the IEEE international
  conference on computer vision. pp. 5058--5066 (2017)

\bibitem{malach2020proving}
Malach, E., Yehudai, G., Shalev-Schwartz, S., Shamir, O.: Proving the lottery
  ticket hypothesis: Pruning is all you need. In: International Conference on
  Machine Learning. pp. 6682--6691. PMLR (2020)

\bibitem{morcos2019one}
Morcos, A., Yu, H., Paganini, M., Tian, Y.: One ticket to win them all:
  generalizing lottery ticket initializations across datasets and optimizers.
  Advances in neural information processing systems  \textbf{32} (2019)

\bibitem{orseau2020logarithmic}
Orseau, L., Hutter, M., Rivasplata, O.: Logarithmic pruning is all you need.
  Advances in Neural Information Processing Systems  \textbf{33},  2925--2934
  (2020)

\bibitem{otter2020survey}
Otter, D.W., Medina, J.R., Kalita, J.K.: A survey of the usages of deep
  learning for natural language processing. IEEE transactions on neural
  networks and learning systems  \textbf{32}(2),  604--624 (2020)

\bibitem{pouyanfar2018survey}
Pouyanfar, S., Sadiq, S., Yan, Y., Tian, H., Tao, Y., Reyes, M.P., Shyu, M.L.,
  Chen, S.C., Iyengar, S.S.: A survey on deep learning: Algorithms, techniques,
  and applications. ACM Computing Surveys (CSUR)  \textbf{51}(5),  1--36 (2018)

\bibitem{simonyan2014very}
Simonyan, K., Zisserman, A.: Very deep convolutional networks for large-scale
  image recognition. arXiv preprint arXiv:1409.1556  (2014)

\bibitem{srivastava2014dropout}
Srivastava, N., Hinton, G., Krizhevsky, A., Sutskever, I., Salakhutdinov, R.:
  Dropout: a simple way to prevent neural networks from overfitting. The
  journal of machine learning research  \textbf{15}(1),  1929--1958 (2014)

\bibitem{tessera2021keep}
Tessera, K.a., Hooker, S., Rosman, B.: Keep the gradients flowing: Using
  gradient flow to study sparse network optimization. arXiv preprint
  arXiv:2102.01670  (2021)

\bibitem{wang2020picking}
Wang, C., Zhang, G., Grosse, R.: Picking winning tickets before training by
  preserving gradient flow. arXiv preprint arXiv:2002.07376  (2020)

\bibitem{wen2016learning}
Wen, W., Wu, C., Wang, Y., Chen, Y., Li, H.: Learning structured sparsity in
  deep neural networks. Advances in neural information processing systems
  \textbf{29} (2016)

\bibitem{wunder2021reverse}
Wunder, G., Gross, B., Fritschek, R., Schaefer, R.F.: A reverse jensen
  inequality result with application to mutual information estimation. In: 2021
  IEEE Information Theory Workshop (ITW). pp.~1--6. IEEE (2021)

\bibitem{xiao2017fashion}
Xiao, H., Rasul, K., Vollgraf, R.: Fashion-mnist: a novel image dataset for
  benchmarking machine learning algorithms. arXiv preprint arXiv:1708.07747
  (2017)

\bibitem{you2019drawing}
You, H., Li, C., Xu, P., Fu, Y., Wang, Y., Chen, X., Baraniuk, R.G., Wang, Z.,
  Lin, Y.: Drawing early-bird tickets: Towards more efficient training of deep
  networks. arXiv preprint arXiv:1909.11957  (2019)

\bibitem{you2019gate}
You, Z., Yan, K., Ye, J., Ma, M., Wang, P.: Gate decorator: Global filter
  pruning method for accelerating deep convolutional neural networks. Advances
  in neural information processing systems  \textbf{32} (2019)

\bibitem{zhou2018online}
Zhou, Z., Zhou, W., Hong, R., Li, H.: Online filter weakening and pruning for
  efficient convnets. In: 2018 IEEE International Conference on Multimedia and
  Expo (ICME). pp.~1--6. IEEE (2018)

\bibitem{zullich2021speeding}
Zullich, M., Medvet, E., Pellegrino, F.A., Ansuini, A.: Speeding-up pruning for
  artificial neural networks: introducing accelerated iterative magnitude
  pruning. In: 2020 25th International Conference on Pattern Recognition
  (ICPR). pp. 3868--3875. IEEE (2021)

\end{thebibliography}
